\documentclass[10pt]{article}
\usepackage[numbers]{natbib}
\usepackage{microtype}
\usepackage{graphicx}
\usepackage{subfigure}
\usepackage{booktabs} 
\usepackage{amsmath, amsthm, amssymb, bm,amstext}
\usepackage{multirow}
\usepackage{url}
\usepackage{makecell}

\newtheorem{thm}{Theorem}

\newtheorem{lemma}{Lemma}
\newtheorem{cor}{Corollary}
\newtheorem{definition}{Definition}

\usepackage{hyperref}

\usepackage[body={6.5in,8in}]{geometry}

\title{CCMN: A General Framework for Learning with Class-Conditional Multi-Label Noise}

\date{}

\author{
Ming-Kun Xie and Sheng-Jun Huang\thanks{Correspondence to: Sheng-Jun Huang}\\
\\
  College of Computer Science and Technology, \\
  Nanjing University of Aeronautics and Astronautics\\
  MIIT Key Laboratory of Pattern Analysis and Machine Intelligence\\
  Collaborative Innovation Center of Novel Software Technology and Industrialization, \\
  Nanjing, 211106\\
  \texttt{\{mkxie,huangsj\}@nuaa.edu.cn}
}

\begin{document}

\def \p {\bm{p}}
\def \Rd {\mathbb{R}^{d}}
\def \x {\bm{x}}
\def \y {\bm{y}}
\def \ls {\{0,1\}^{q}}
\def \fs {\mathbb{R}^{d\times n}}
\def \X {\mathbf{X}}
\def \Y {\mathbf{Y}}
\def \tildeY {\widetilde{\mathbf{Y}}}
\def \hatV {\bm{\widehat{V}}}
\def \b {\mathbf{b}}
\def \s {\mathbf{s}}
\def \V {\bm{V}}
\def \bphi {\bm{\phi}}
\def \v {\bm{v}}
\def \Vstar {\V^{*}}
\def \L {\mathcal{L}}
\def \tL {\tilde{\mathcal{L}}}
\def \W {\bm{W}}
\def \U {\bm{U}}
\def \hatU {\widehat{\U}}
\def \hatv {\widehat{\v}}
\def \hatu {\widehat{\u}}
\def \t {\mathbf{t}}
\def \u {\bm{u}}
\def \r {{\mathbf{rank}}}
\def \tr {{\rm tr}}
\def \bPhi {\bm{\Phi}}
\def \Sig {\bm{\Sigma}}
\def \F {{\mathcal{F}}}
\def \Z {\bm{Z}}
\def \Lam {\mathbf{\Lambda}}
\def \argmin {{\rm argmin}}
\def \prox {\mathbf{prox}}
\def \E {\mathbb{E}}
\def \H {\bm{H}}
\def \hell {\hat{\ell}}
\def \ty {\tilde{\bm{y}}}
\def \hf {\hat{f}}
\def \hR {\widehat{R}}
\def \D {\mathcal{D}}
\def \pred {{\rm{pred}}}
\def \tell {\tilde{\ell}}
\def \R {\mathcal{R}}
\def \h {\bm{h}}
\def \tf {\tilde{f}}
\def \tLova {\tilde{\mathcal{L}}_{\rm{OVA}}}
\def \Lova {\mathcal{L}_{\rm{OVA}}}
\def \tLpc {\tilde{\mathcal{L}}_{\rm{PC}}}
\def \Lpc {\mathcal{L}_{\rm{PC}}}
\def \Lr {\mathcal{L}_{r}}
\def \Lh {\mathcal{L}_{h}}
\def \tLr {\tilde{\mathcal{L}}_{r}}
\def \tLh {\tilde{\mathcal{L}}_{h}}
\def \tR {\tilde{R}}
\def \tphi {\tilde{\phi}}
\def \tpsi {\tilde{\psi}}
\def \tPhi {\tilde{\Phi}}
\def \Pr {{\rm{Pr}}}
\def \f {\bm{f}}

\maketitle

\begin{abstract}
	Class-conditional noise commonly exists in machine learning tasks, where the class label is corrupted with a probability depending on its ground-truth. Many research efforts have been made to improve the model robustness against the class-conditional noise. However, they typically focus on the single label case by assuming that only one label is corrupted. In real applications, an instance is usually associated with multiple labels, which could be corrupted simultaneously with their respective conditional probabilities. In this paper, we formalize this problem as a general framework of learning with Class-Conditional Multi-label Noise (CCMN for short). We  establish two unbiased estimators with error bounds for solving the CCMN problems, and further prove that they are consistent with commonly used multi-label loss functions. Finally, a new method for partial multi-label learning is implemented with unbiased estimator under the CCMN framework. Empirical studies on multiple datasets and various evaluation metrics validate the effectiveness of the proposed method.
	
\end{abstract}

\section{Introduction}

In ordinary supervised classification problems, a common assumption is that the class labels of training data are always correct. However, in practice, the assumption hardly holds since the training examples are usually corrupted due to unavoidable reasons, such as measurement error, subjective labeling bias or human labeling error. Learning in presence of label noise has been a problem of theoretical as well as practical interest in machine learning communities \cite{FB14}. In various applications, such as the image annotation \cite{XT15} and text classification \cite{jindal2019noisetext}, many successful methods have been applied to learning with label noise.

In general, to deal with noise-corrupted data, it is crucial to discover the causes of label noise. A natural and simple formulation of label noise is that the labels are corrupted by a random noise process \cite{BB11}, which can be described by the random classification noise (RCN) framework \cite{AD88}. RCN framework assumes that each label is flipped independently with a specific probability $\rho\in[0,\frac{1}{2})$. Despite the great impact that RCN framework has made, it is too simple to deal with practical tasks, where the cause of label noise may not follow a random process. In order to deal with such problems, in \cite{SG09,NND13}, authors propose a class-conditional random label noise (CCN) framework, where the probability of a label flipped depends on its true label. Unfortunately, CCN framework only considers single label case and fail to deal with multiple noisy labels, where multiple labels assigned to one instance may be corrupted simultaneously.

In this paper, we extend CCN to a more general framework of Class Conditional Multi-label Noise (CCMN) to learn with multi-class and multi-label classification tasks. In CCMN framework, each instance is associated with multiple labels and each of class labels may be flipped with a probability $\rho^{j}_{+1}$ or $\rho^{j}_{-1}$, which depends on its true label $y_{j}\in\{+1,-1\}$. It is worth noting that a great deal of real-world applications in weak-supervised setting \cite{ZZH18} can be included into the framework, such as learning from partial-labeled data \cite{TB11,MK18} or learning with missing labels \cite{YY10}, which we will discuss detailedly in following content. To the best of our knowledge, general theoretical results in this setting have not been developed before.

To tackle corrupted data with class-conditional multiple noisy labels, we derive a modified loss function for learning a multi-label classifier with risk minimization. Theoretically, we show that the empirical risk minimization with modified loss functions can be in an unbiased fashion from independent and dependent perspectives. We then provide the estimation error bound for the two unbiased estimator and further show that learning with class-conditional multiple noisy labels can be multi-label consistent to two commonly used losses, i.e., hamming loss and ranking loss, respectively. Finally, CCMN can be regarded as a general framework of various weakly-supervised learning scenarios, such as partial label learning \cite{TB11}, partial multi-label learning \cite{MK18} and weak label learning \cite{YY10}. Among them, partial multi-label learning (PML) is a recently proposed weakly-supervised learning scenario. We further propose a new approach for partial multi-label learning with unbiased estimator. Comprehensive  empirical studies demonstrate the effectiveness of the proposed methods.


Our main contributions are summarized as follows:
\begin{itemize}
	\item A general framework of learning from class-conditional multi-label noise is proposed. Varied weakly-supervised learning scenarios can be cast under CCMN framework.
	\item We propose two unbiased estimators for solving CCMN problems in independent and dependent fashion. These two estimators are proven to be multi-label consistent to hamming loss and ranking loss, respectively. 
	\item A novel approach for Partial Multi-label Learning with unbiased estimator (uPML) is proposed under the CCMN framework. 
\end{itemize}

\section{Related Works}
There are a great deal of previous works aiming to learn a robust classification model in presence of label noise. Most of these methods only consider single label case and cannot tackle class-conditional multi-label noise. 

There has been a long line of studies in machine learning community on random label noise. The earliest work by Angluin and Laird \cite{AD88} first propose the \textit{random classification noise} (RCN) model. In \cite{FY09}, authors propose a boosting-based method and empirically show its robustness to random label noise. A theoretical analysis proposed in \cite{LPM10} proves that any method based on convex surrogate loss is inherently ill-suited to random label noise. In \cite{BB11}, authors consider random and adversarial label noise and try to utilize a robust SVM to deal with label noise. 


The first attempt to tackle class-conditional label noise is in \cite{SG09} where authors propose a variant of SVM to deal with noisy labels with theoretical guarantee. In \cite{NND13}, authors propose unbiased estimators to solve class-conditional label noise in binary classification setting. Based on the assumption that the class-conditional distributions may overlap, a method called mixture proportion estimation is proposed in \cite{SC13} to estimate the maximal proportion of one distribution that is present in another. Furthermore, authors extend the method into multi-class setting in \cite{BG14}. 


Thanks to the great development of deep learning, there are various methods raised for utilizing deep neural networks to handle noisy labels, such as label correction methods \cite{VA17,LYC17}, loss correction methods \cite{HC19,XHW18,HB18}, sample reweighting methods \cite{KMP10,ren2018l2rw,xia2019anchor} and robust loss methods \cite{GA17,ZZ18,LYM19}.

In this paper, we also employ the CCMN framework to solve partial multi-label learning problems. In order to deal with partial-labeled data, the most commonly used strategy is disambiguation \cite{TB11,gong2017pll}, which recovers ground-truth labeling information for candidate labels. Some methods perform disambiguation strategy by estimating a confidence for each candidate label \cite{MK18,zhang2020pml}. Other methods utilize decomposition scheme \cite{LS19} or adversarial training \cite{yan2019apml}. However, aforementioned methods never consider the generation process of noisy labels in candidate label set, which is a essential information for solving PML problems. In \cite{xie2020pmlni}, authors first consider modeling the relationship between noisy labels and feature representation. Nevertheless, all of these methods failed to prove the consistency of the proposed method.




\section{Formulation of CCMN}


Let $\x\in\mathcal{X}$ be a feature vector and $\y\in\mathcal{Y}$ be its corresponding label vector, where $\mathcal{X}\subset\mathbb{R}^{d}$ is the feature space and $\mathcal{Y}\subset\{-1,1\}^{q}$ is the target space with $q$ possible class labels. For notational convenience, the label $y_{j}$ can be denoted by its index $j$. In the setting, $y_{j}=1$ indicates the $j$-th label is a true label for instance $\x$; $y_{j}=-1$, otherwise. In this paper, we focus on the multi-label learning problem, where each instance may be assigned with more than one label, i.e., $\sum_{j=1}^{q}I(y_ {j}=1)\geq 1$ holds, where $I$ is the indicator function. Let $S=\{(\x_{1},\y_{1}),...,(\x_{n},\y_{n})\}$ be the given training data set, drawn \textit{i.i.d.} according to the true distribution $\D$. We also use $[q]$ to denote the integer set $\{1,...,q\}$.

In this paper, for instance $\x$, its corresponding label $\y$ can be corrupted and may be flipped into $\ty$ following a class-conditional multi-label noise model as follows:
\begin{eqnarray*}
	&\Pr(\tilde{y}_{j}=-1|{y}_{j}=+1)=\rho_{+1}^{j},\\ &\Pr(\tilde{y}_{j}=+1|{y}_{j}=-1)=\rho_{-1}^{j},\\
	& \forall j\in[q], \rho_{+1}^{j}+\rho_{-1}^{j}< 1.
\end{eqnarray*}
where $\rho_{+1}^{j}$ and $\rho_{-1}^{j}$ are noise rates for each class label and are assumed to be known to the learner. 



After injecting random noise into original samples $S$, the observed data set $S_{\rho}=\{(\x_{1},\ty_{1}),...,(\x_{n},\ty_{n})\}$ are drawn \textit{i.i.d.}, according to distribution $\D_{\rho}$. Our goal is to learn a prediction function $\h:\mathcal{X}\rightarrow\mathcal{Y}$ can accurately predict labels for any unseen instance. In general, it is not easy to learn $\h$ directly, and alternatively, one usually learns a real-valued decision function $\f:\mathcal{X}\rightarrow\mathbb{R}^{K}$. Note that, for each instance $\x$, even though its final prediction depends on $\h(\x)$, we also call $\f$ itself the classifier. As mentioned before, CCMN can be used to solve multi-class and multi-label learning problems. In this paper, we focus on the multi-label classification task without loss of generality.

In general, CCMN framework has implications in varied
real-world applications. In the following content, we discuss on two popular weakly-supervised learning scenarios, i.e., partial multi-label learning and weak label learning, which can be regarded as special cases of CCMN framework. 

In PML problems, each instance is associated with a candidate label set $Y_{\rm{c}}$, which contains multiple relevant labels and some other noisy labels. One intuitive strategy to solve the task is that treat all labels in $Y_{c}$ as relevant labels and transform $Y_{c}$ into $\tilde{Y}$. Here, we use $\tilde{Y}$, since there may exist irrelevant labels in $\tilde{Y}$. Then, a new training set is obtained, where each instance is associated with the label set $\tilde{Y}$. Besides relevant labels, $\tilde{Y}$ is also injected into some irrelevant labels, which can be regarded as  class-conditional multi-label noise, i.e., irrelevant labels are flipped into relevant labels  with $\rho_{-1}^{j}>0$ while $\rho_{+1}^{j}=0$. 

In weak label learning, also known as learning with missing label,  relevant labels of each instance are partially known. Specifically, each instance is associated with a label set $Y_{+1}$ while $\tilde{Y}_{-1}$ is the irrelevant label set. Here, we use $\tilde{Y}_{-1}$ since there may exist missing labels in $\tilde{Y}_{-1}$, i.e., relevant labels are missed. Similarly, one can treat all labels in $\tilde{Y}_{-1}$ as irrelevant labels. Besides irrelevant labels, $\tilde{Y}_{-1}$ is also injected into some relevant labels which can be regarded as  class-conditional multi-label noise with $\rho_{+1}^{j}>0$ while $\rho_{-1}^{j}=0$. 

\section{Learning with CCMN}
\label{sec:CCMN}
In this section, we first provide some necessary preliminaries, and then derive two CCMN solvers in independent and dependent cases. In the independent case, we solve each binary classification task of the CCMN problem independently while in the dependent case, we solve the CCMN problem by considering label correlations.

\subsection{Preliminaries}
To derive our results for solving CCMN problems, we introduce some notations and the property of multi-label consistency.

There are many multi-label loss functions (also called evaluation metrics), such as \textit{hamming loss}, \textit{ranking loss} and \textit{average precision} \cite{zhang2013review}, etc. In this paper, we focus on two well known loss functions, i.e., hamming loss and ranking loss, and leave the discussion on other loss functions in the future work. 

The hamming loss considers how many instance-label pairs are misclassified. Given the decision function $\f$ and prediction function $F$, the hamming loss can be defined by:
\begin{equation}
L_{h}(F(\f(\x)),\y) = \frac{1}{q}\sum_{j=1}^{q}I[\hat{y}_j\neq y_j]
\end{equation}
where $\hat{\y}=F(\f(\x))=(\hat{y}_1,...,\hat{y}_q)$. 

The ranking loss considers label pairs that are ordered reversely for an instance. Given a real-value decision function $\f=(f_1,f_2,...,f_q)$, the ranking loss can be defined by:
\begin{equation}
L_{r}(\f,\y) = \sum_{1\leq j<k\leq q}I(y_{j}<y_{k})\ell(j,k)+I(y_{j}>y_{k})\ell(k,j)
\end{equation}
where
\begin{equation*}
\ell(j,k) = I(f_j>f_k)+\frac{1}{2}I(f_j=f_k)
\end{equation*}

The risk of $\f$ with respective to loss $L$ is given by $R(\f)=\E_{(\x,\y)\sim\D}[L(\f(\x),\y)]$ and the minimal risk (also called the Bayes risk) can be defined by $R^{*}=\inf_{\f}R(\f)$. For an instance $\x$, the conditional risk of $\f$ can be defined as
\begin{equation*}
l(\p,\f) = \sum_{\y\in\mathcal{Y}}p_{\y}L(\f,\y)
\end{equation*}
where $p_{\y}=(\Pr(\y|\x))_{\y\in\mathcal{Y}}$ is a vector of conditional probability over $\y\in\mathcal{Y}$.

Note that the above mentioned two loss functions are non-convex and computationally NP-hard, which makes the corresponding optimization problems hard to solve. In practice, a feasible solution is to consider alternatively a surrogate loss function $\L$ which can be optimized efficiently. The $\L$-risk and minimal $\L$-risk can be defined as $R_{\L}(\f)=\E_{(\x,\y)\sim\D}[\L(\f(\x),\y)]$ and $R_{\L}^*=\inf_{\f}R_{\L}(\f)$, respectively. Accordingly, we define the empirical $\L$-risk as$\hR_{\L}(\f)=\frac{1}{n}\sum_{i=1}^{n}\L(\f(\x),\y)$.

Furthermore, we define the conditional $\L$-risk of $\f$ and the conditional Bayes $\L$-risk
\begin{equation*}
W(\p,\f) = \sum_{\y\in\mathcal{Y}}p_{\y}\L(\f,\y),\quad W^{*}(\p)=\inf_{\f}W(\p,\f).
\end{equation*}
The definition of multi-label consistency can be formulated as follows.
\begin{definition}
	\cite{gao2013consistency} Given a below-bounded surrogate loss $\L$, where $\L(\cdot, \y)$ is continuous for every $\y\in\mathcal{Y}$, $\L$ is said to be multi-label consistent w.r.t. the loss $L$ if it holds, for every $\p$, that 
	\begin{equation*}
	W^{*}(\p)<\inf_{\f}\{W(\p,\f):\f\notin\mathcal{A}(\p)\},
	\end{equation*}
	where $\mathcal{A}(\p)=\{\f:l(\p,\f)=\inf_{\f'}l(\p,\f')\}$ is the set of Bayes decision functions. 
\end{definition}
Based on the definition, the following theorem can be further established.
\begin{thm}
	\cite{gao2013consistency} The surrogate loss $\L$ is multi-label consistent w.r.t. the loss $L$ if and only if it holds for any sequence $\{\f_n\}_{n\geq1}$ that 
	\begin{equation*}
	\text{if} \hspace{1em} R_{\L}(\f_n)\rightarrow R_{\L}^* \hspace{1em} \text{then} \hspace{1em} R(\f_n)\rightarrow R^*.
	\end{equation*}
\end{thm}
The theorem tells us that the multi-label consistency is a necessary and sufficient condition for the convergence of $\L$-risk to the Bayes $\L$-risk, implying $R(\f)\rightarrow R^{*}$. Our goal is to learn a good classification model with the modified loss function $\tL(\f(\x),\ty)$ from noise-corrupted data by minimizing empirical $\tL$-risk:
\begin{equation*}
\hat\f=\arg\min\limits_{\f\in\mathcal{F}} \hR_{\tL}(\f),
\end{equation*}
where $\mathcal{F}$ is a function class.


\subsection{Independent Case}

In order to solve multi-label learning problems, the most straightforward method is to decompose the task into $q$ independent binary classification problems \cite{zhang2013review}, where the goal is to learn $q$ functions, $\f=(f_1,f_2,...,f_q)$. However, as mentioned before, it is difficult to directly optimize the hamming loss due to its non-convexity. Alternatively, we consider the following surrogate loss
\begin{equation}\label{eq:hl}
\L_h(\f(\x),\y)=\sum_{j=1}^{q}\phi(y_{j}f_{j}(\x)),
\end{equation}
where $\phi$ is a convex loss function. The common choices are least square loss $\phi(t)=(1-t)^2$ and hinge loss $\phi(t)=(1-t)_{+}$ in \cite{elisseeff2002rankingsvm}. 


The modified loss function under class-conditional multi-label noise in independent case can be defined as follows:
\begin{equation}\label{eq:hl_ue}
\tLh(\f,\ty)=\sum_{j=1}^{q}\tphi(y_jf_j(\x)),
\end{equation}
where,
\begin{equation*}
\tphi(y_jf_j) = \kappa_{j}\left[(1-\rho_{-y_j})\phi(y_jf_j)-\rho_{y_j}\phi(-y_jf_j)\right].
\end{equation*}
Here, $\kappa_{j}=\frac{1}{1-\rho_{+1}^{j}-\rho_{-1}^{j}}$ is a constant. 

We extend the results in \cite{NND13} to have the following lemma, which shows the unbiasedness of the estimator defined by Eq.(\ref{eq:hl_ue}).
\begin{lemma}
	For any $y_j,\forall j\in[q]$, let $\phi(\cdot)$ be any bounded loss function. Then, 
	if $\tLh(\f,\ty)$ can be defined by Eq.(\ref{eq:hl_ue}), we have $\E_{\ty}\left[\tLh(\f,\ty)\right]=\Lh(\f,\y)$.
\end{lemma}

Let $\bm{\sigma}=\{\sigma_{1},...,\sigma_{n}\}$ be $n$ Rademacher variables with $\sigma_{i}$ independently uniform random variable taking value in $\{-1,+1\}$. Then, the Rademacher complexity with respective to function class $\mathcal{F}$ and unbiased estimator $\tL$ can be formulated as follows:
\begin{equation*}
\R_{n}(\tL\circ\F)=\E_{\x,\ty,{\bm\sigma}}\left[\sup_{f\in{\F}}\frac{1}{n}\sum_{i=1}^{n}\sigma_{i}\tL(\f(\x_{i}),\ty_{i})\right].
\end{equation*}

Accordingly, the performance guarantee for the unbiased estimator in independent case can be derived as following theorems.
\begin{thm}\label{thm:hl_ue_bound} 
	Let $\mu=\max_{j}\frac{1}{1-\rho_{-1}^j-\rho_{+1}^j},\forall j\in[q]$. Then, for the loss function $\phi(\cdot)$ bounded by $\Theta$, with probability at least $1-\delta$, we have
	\begin{equation*}
	R_{\L_h}(\hat\f)-\min_{\f\in\mathcal{F}}R_{\L_h}(\f)\leq 4qK_{\rho}\R_n(\F)+2q\mu\Theta\sqrt{\frac{\ln\frac{1}{\delta}}{2n}},
	\end{equation*}
	where $\hat\f$ is trained by minimizing $\hR_{\tLh}(\f)$ and $K_{\rho}$ is the Lipschitz constant of $\tLh$.
\end{thm}

\begin{thm}\label{thm:hl_ue_con}
	If $\phi$ is convex function with $\phi'(0)<0$, then the modified surrogate loss $\tL_h$ defined by Eq.(\ref{eq:hl_ue}) is consistent w.r.t hamming loss, i.e., there exists a non-negative concave function $\xi$ with $\xi(0)=0$, such that
	\begin{equation*}
	R_{L_h}(\hat\f)-R_{L_h}^{*}\leq\xi(R_{\L_{h}}(\hat\f)-R_{\L_{h}}^{*}).
	\end{equation*}
\end{thm}

As shown in Theorem \ref{thm:hl_ue_bound}, the generalization error is dependent to the noise rates $\rho_{-1}$ and $\rho_{+1}$. It is intuitive that smaller noise rates lead to a better generalization performance owing to a smaller $\mu$. 
By combining Theorem \ref{thm:hl_ue_bound} with Theorem \ref{thm:hl_ue_con},  we obtain a performance guarantee for learning from class-conditional multiple noisy labels with respective to hamming loss, that is as $n\rightarrow \infty$, we have $R_{\L_h}(\hat\f_n)\rightarrow R_{\L_h}^*$, then $R_{L_h}(\hat\f_n)\rightarrow R_{L_h}^*$.

\subsection{Dependent Case}

In multi-label learning problems, a common assumption is that there exist label correlations among labels \cite{read2011cc,HP14}. Therefore, it is fundamental to learning with class-conditional multi-label noise in a dependent fashion. The ranking loss considers the second-order label correlation and its surrogate loss is commonly defined as
\begin{align}\label{eq:rl}
\quad\Lr(\f(\x),\y)=\sum_{1\leq j<k\leq q}I(y_j>y_k)\phi(f_{jk})+I(y_j<y_k)\phi(f_{kj})=\sum_{1\leq j<k\leq q}\phi(y_{jk}(f_j-f_k))
\end{align}
where $y_{jk}=\frac{y_j-y_k}{2}$ and $f_{jk}=f_j-f_k$.

If $y_{j}\neq y_k$, let $a=(1-\rho_{-y_{jk}}^j)(1-\rho_{y_{jk}}^{k}), b=\rho_{y_{jk}}^{j}\rho_{-y_{jk}}^{k}$
and if $y_{j}=y_k$, let $
c = \rho_{y_j}^{j}(1-\rho_{-y_k}^{k}), d=\rho_{y_k}^k(1-\rho_{-y_j}^{j})$. Then, the modified loss function under class-conditional multi-label noise in the dependent case can be defined as 
\begin{equation}\label{eq:rl_ue}
\tLr(\f(\x),\y)=\sum_{1\leq j<k\leq q}\tphi((f_j,f_k),(y_j,y_k))
\end{equation}
where $\tphi((f_j,f_k),(y_j,y_k))$ can be abbreviated by $\tphi(j,k)$,
\begin{eqnarray*}
	\tphi(j,k)=\left\{
	\begin{aligned}
		&\kappa_{jk}\left[a\phi(y_{jk}f_{jk})+b\phi(-y_{jk}f_{jk})\right]    &\text{if}\ y_{j}\neq y_{k}\\
		&-\kappa_{jk}\left[c\phi(-y_jf_{jk})+d\phi(y_jf_{jk})\right] &\text{if}\ y_{j}=y_k
	\end{aligned}
	\right.	
\end{eqnarray*}
Here, $\kappa_{jk}=\frac{1}{(1-\rho_{+1}^{j}-\rho_{-1}^{j})(1-\rho_{+1}^{k}-\rho_{-1}^{k})}$ is a constant.


The unbiasedness of the estimator defined by Eq.(\ref{eq:rl_ue}) can be shown as following lemma.
\begin{lemma}
	For any $y_j,y_k,\forall j,k\in[q]$, let $\phi(\cdot)$ be any bounded loss function. Then, 
	if $\tLr(\f,\ty)$ can be defined by Eq.(\ref{eq:rl_ue}), we have $\E_{\ty}\left[\tLr(\f,\ty)\right]=\Lr(\f,\y)$.
\end{lemma}

Accordingly, the  performance guarantee for the unbiased estimator in the dependent case can be derived as following theorems. 
\begin{thm}\label{thm:rl_ue_bound}
	Let $\mu=\max_j\frac{1+|\rho_{-1}^j-\rho_{+1}^j|}{(1-\rho_{-1}^j-\rho_{+1}^j)^{2}},\forall j\in[q]$. Then, for the loss function $\phi(\cdot)$ bounded by $\Theta$, with probability at least $1-\delta$, we have
	\begin{equation*}
	\quad R_{\L_r}(\hat\f)-\min_{\f\in\mathcal{F}}R_{\L_r}(\f)\leq 4q(q-1)K_{\rho}\R_{n}(\F)+2q(q-1)\mu\Theta\sqrt{\frac{\ln\frac{1}{\delta}}{2n}}
	\end{equation*}
	where $\hat\f$ is trained by minimizing $\hR_{\tLr}(\f)$ and $K_{\rho}$ is the Lipschitz constant of $\tLr$.
\end{thm}

\begin{thm}\label{thm:rl_ue_con}
	If $\phi$ is a differential and non-increasing function with $\phi'(0)<0$ and $\phi(t)+\phi(-t)=2\phi(0)$, then the modified surrogate loss $\tL_r$ defined by Eq.(\ref{eq:rl_ue}) is consistent w.r.t ranking loss, i.e., there exists a non-negative concave function $\xi$ with $\xi(0)=0$, such that
	\begin{equation}
	R_{L_r}(\hat\f)-R_{L_r}^{*}\leq\xi(R_{\L_{r}}(\hat\f)-R_{\L_{r}}^{*})
	\end{equation}
\end{thm}
From Theorem \ref{thm:hl_ue_bound}, it can be observed that  besides the noise rates $\rho_{-1}$ and $\rho_{+1}$, the generalization error is also dependent to the difference between noise rates, i.e.,  $\left|\rho_{-1}-\rho_{+1}\right|$. Generally, a smaller difference also leads to a better generalization performance owing to a smaller $\mu$.
By combining Theorem \ref{thm:hl_ue_bound} with Theorem \ref{thm:hl_ue_con},  we obtain a performance guarantee for learning from class-conditional multiple noisy labels with respective to ranking loss, that is as $n\rightarrow \infty$, we have $R_{\L_r}(\hat\f_n)\rightarrow R_{\L_r}^*$, then $R_{L_r}(\hat\f_n)\rightarrow R_{L_r}^*$.

\textbf{Testing Phase}. Regarding the classifier $\hat\f$ trained by minimizing the loss function $\tLh$ or $\Lh$, for each testing instance $\x_t$, we use the ${\rm{sgn}}(f_j(\x_t))$ to predict its labels, where ${\rm{sgn}}(a)$ is a function which outputs $+1$ if $a>=0$; $-1$, otherwise. 

However, regarding classifier $\f$ trained by minimizing the loss function $\tLr$ or $\Lr$, it is unreasonable to use 0 as a threshold directly. Instead, to perform predictions, in training phase, we introduce a \textit{dummy} label $y_{0}=0$ for each instance, and then, in testing phase, the output of the classifier for label $y_{0}$ is used as the threshold to decide the label assignment for each instance. 

Specifically, for instance $\x$, suppose that its corresponding label vector can be represented by $\y=(y_{0},y_{1},...,y_{q})$, where $y_{0}$ is indexed by 0. In the training phase, the cumulative loss $\tL_0$ with respective to $y_{0}$ can be written as follows:
\begin{equation}\label{eq:dummy}
\tL_0=\sum_{j=1}^{q}\tphi(f_{j0}(\x),y_{j0})
\end{equation}
where $f_{j0}(\x)=f_{j}(\x)-f_{0}(\x)$ and $y_{j0}=y_{j}-y_{0}$. Note that $y_{j0}=y_{j}$, since we have $y_{0}=0$. In the situation, the formulation of loss $\tL_0$ is the same as Eq.(\ref{eq:hl_ue}) and can be solved efficiently.

\section{Partial Multi-Label Learning with CCMN}

In the section, from perspective of CCMN framework, we propose a new approach for partial multi-label learning with unbiased estimator (uPML for short). In such case, $\forall j\in[q], \Pr(\tilde{y}_{j}=+1|y_{j}=-1)=\rho^{j},\Pr(\tilde{y}_j=-1|y_{j}=+1)=0$ holds. Here, we omit the subscript of $\rho^j$ for notational simplicity.

Based on the results in Section \ref{sec:CCMN}, for the independent case, the objective function of uPML $\tL_h^{\text{PML}}$ can be defined as follows:
\begin{equation}
\tL_h^{\text{PML}}(\f(\x),\y)=\sum_{j=1}^{q}\tphi(y_{j}f_{j}(\x))
\end{equation}
where,
\begin{eqnarray*}
	\tphi(y_jf_j(\x))=\left\{
	\begin{aligned}
		& \frac{\phi(-f_j(\x))-\rho^j\phi(f_j(x))}{1-\rho^{j}}  &\text{if}\ y_{j}= -1\\
		& \phi(f_j(x))  &\text{otherwise}
	\end{aligned}
	\right.	
\end{eqnarray*}

For the dependent case, the objective function of uPML $\tL_r^{\text{PML}}$  can be defined as follows:
\begin{equation}
\tL_r^{\text{PML}}(\f(\x),\y)=\sum_{1\leq j<k\leq q}\tphi(y_{jk}f_{jk}(\x))
\end{equation}
where,
\begin{equation*}
\tphi(y_{jk}f_{jk})=\left\{
\begin{array}{lll}
&\frac{\phi(f_{jk})}{1-\rho^{k}}  &\text{if}\ y_{jk}= +1\\
& \frac{\phi(-f_{jk})}{1-\rho^{j}}   &\text{if}\ y_{jk}= -1\\
& \frac{-\rho^j\phi(f_{jk})-\rho^k\phi(-f_{jk})}{(1-\rho^j)(1-\rho^k)} &\text{if}\ y_{j}=y_k=-1\\
& 0 &\text{if}\ y_{j}=y_k= +1
\end{array}
\right.	
\end{equation*}

We derive the generalization error bound for the proposed uPML method, which is the special case of Theorem \ref{thm:hl_ue_bound} and \ref{thm:rl_ue_bound}.

\begin{cor} \label{thm:bound} 
	For the least square loss $\phi(t)=(1-t)^2$, with probability at least $1-\delta$, we have
	\begin{equation*}
	R_{\L_h}(\hat\f)-\min_{\f\in\mathcal{F}}R_{\L_h}(\f)\leq 4qK_{\rho}\R_n(\F)+\frac{8q}{1-\rho_{\min}}\sqrt{\frac{\ln\frac{1}{\delta}}{2n}}
	\end{equation*}
	where $\rho_{\min}=\min_j\rho^j$ and $\hat\f$ is trained by minimizing $\hR_{\tLh^{\text{PML}}}(f)$, and
	\begin{equation}\quad R_{\L_r}(\hat\f)-\min_{\f\in\mathcal{F}}R_{\L_r}(\f)\leq 4q(q-1)K_{\rho}\R_n(\F)+8q(q-1)\mu\sqrt{\frac{\ln\frac{1}{\delta}}{2n}}
	\end{equation}
	where $\mu=\min_{j}\frac{1+\rho^j}{(1-\rho^j)^2}$ and $\hat\f$ is trained by minimizing $\hR_{\tLr^{\text{PML}}}(\f)$.
\end{cor}

\begin{table*}
	\centering
	\tiny
	\caption{Comparison results among $\tLh$, $\tLr$ and their baselines with diverse loss functions using the linear model,  where $\bullet$/$\circ$ indicates whether the proposed method is significantly superior/inferior to the comparing  methods via paired $t$-test (at 0.05 significance level).}\label{table:linear_res}
	\renewcommand{\multirowsetup}{\centering}
	\begin{tabular}{l|l l l l l l l l l}
		\hline
		Data& \makecell[c]{yeast} & \makecell[c]{enron} & \makecell[c]{scene} & \makecell[c]{slashdot} & \makecell[c]{music\_emotion} & \makecell[c]{music\_style} & \makecell[c]{mirflickr} & \makecell[c]{tmc} \\
		\hline
		\multicolumn{8}{l}{Hamming loss (the smaller, the better)}
		\\
		\hline
		$\tL_h/$Square & $\bm{.218} \pm  \bm{.002}$        & $    .058  \pm      .000\bullet$& $    .137  \pm      .004\bullet$& $\bm{.020} \pm  \bm{.000}$        & $    .195  \pm      .001         $& $\bm{.118} \pm  \bm{.000}$        & $    .146  \pm      .000\bullet$& $    .064  \pm      .000\bullet$\\
		$\tL_h/$Hinge & $    .223  \pm      .002\bullet$& $\bm{.057} \pm  \bm{.000}$        & $\bm{.128} \pm  \bm{.003}$        & $    .020  \pm      .001         $& $    .196  \pm      .000\bullet$& $    .125  \pm      .000\bullet$& $    .177  \pm      .001\bullet$& $\bm{.062} \pm  \bm{.000}$        \\
		\hline
		$\tL_r/$Square & $    .231  \pm      .002\bullet$& $    .058  \pm      .001\bullet$& $    .148  \pm      .009\bullet$& $    .027  \pm      .003\bullet$& $    .197  \pm      .001\bullet$& $    .120  \pm      .001         $& $    .139  \pm      .004\bullet$& $    .068  \pm      .000\bullet$\\
		$\tL_r/$Hinge & $    .249  \pm      .004\bullet$& $    .057  \pm      .000         $& $    .133  \pm      .001\bullet$& $    .026  \pm      .003\bullet$& $\bm{.195} \pm  \bm{.000}$        & $    .119  \pm      .001         $& $    .162  \pm      .001\bullet$& $    .064  \pm      .000\bullet$\\
		$\tL_r/$Sigmoid & $    .225  \pm      .002         $& $    .057  \pm      .000         $& $    .130  \pm      .004         $& $    .024  \pm      .002\bullet$& $    .195  \pm      .001         $& $    .122  \pm      .000\bullet$& $\bm{.119} \pm  \bm{.000}$        & $    .068  \pm      .000\bullet$\\
		\hline
		$\L_h/$Square & $    .280  \pm      .017\bullet$& $    .100  \pm      .003\bullet$& $    .163  \pm      .021\bullet$& $    .028  \pm      .004         $& $    .213  \pm      .005\bullet$& $    .124  \pm      .001\bullet$& $    .157  \pm      .003\bullet$& $    .077  \pm      .001\bullet$\\
		$\L_h/$Hinge & $    .265  \pm      .012\bullet$& $    .061  \pm      .002\bullet$& $    .151  \pm      .012\bullet$& $    .023  \pm      .004         $& $    .220  \pm      .001\bullet$& $    .126  \pm      .000\bullet$& $    .185  \pm      .019\bullet$& $    .071  \pm      .000\bullet$\\
		$\L_r/$Square & $    .278  \pm      .016\bullet$& $    .109  \pm      .007\bullet$& $    .163  \pm      .006\bullet$& $    .038  \pm      .001\bullet$& $    .206  \pm      .003\bullet$& $    .157  \pm      .015\bullet$& $    .181  \pm      .008\bullet$& $    .084  \pm      .001\bullet$\\
		$\L_r/$Hinge & $    .271  \pm      .009\bullet$& $    .077  \pm      .005\bullet$& $    .149  \pm      .004\bullet$& $    .034  \pm      .002\bullet$& $    .204  \pm      .002\bullet$& $    .121  \pm      .002         $& $    .198  \pm      .007\bullet$& $    .078  \pm      .001\bullet$\\
		$\L_r/$Sigmoid & $    .265  \pm      .009\bullet$& $    .064  \pm      .001\bullet$& $    .149  \pm      .013\bullet$& $    .035  \pm      .001\bullet$& $    .212  \pm      .001\bullet$& $    .124  \pm      .000\bullet$& $    .138  \pm      .005\bullet$& $    .077  \pm      .001\bullet$\\
		\hline
		\multicolumn{8}{l}{Ranking loss (the smaller, the better)}
		\\
		\hline
		$\tL_h/$MSE & $    .210  \pm      .002\bullet$& $    .216  \pm      .006\bullet$& $    .137  \pm      .014\bullet$& $    .050  \pm      .002         $& $    .237  \pm      .002         $& $    .158  \pm      .002\bullet$& $    .129  \pm      .002\bullet$& $    .061  \pm      .001\bullet$\\
		$\tL_h/$Hinge & $    .237  \pm      .008\bullet$& $    .203  \pm      .006         $& $    .122  \pm      .010         $& $    .049  \pm      .002         $& $    .244  \pm      .002\bullet$& $    .167  \pm      .003\bullet$& $    .150  \pm      .004\bullet$& $\bm{.056} \pm  \bm{.002}$        \\
		\hline
		$\tL_r/$Square & $    .208  \pm      .000\bullet$& $    .221  \pm      .007\bullet$& $    .138  \pm      .011\bullet$& $    .050  \pm      .001         $& $\bm{.237} \pm  \bm{.002}$        & $    .157  \pm      .003\bullet$& $    .114  \pm      .003\bullet$& $    .062  \pm      .002\bullet$\\
		$\tL_r/$Hinge & $    .214  \pm      .003\bullet$& $    .201  \pm      .008         $& $    .120  \pm      .010\bullet$& $    .050  \pm      .001\bullet$& $    .237  \pm      .002         $& $\bm{.153} \pm  \bm{.003}$        & $    .134  \pm      .002\bullet$& $    .058  \pm      .002\bullet$\\
		$\tL_r/$Sigmoid & $\bm{.204} \pm  \bm{.002}$        & $\bm{.201} \pm  \bm{.008}$        & $\bm{.113} \pm  \bm{.009}$        & $\bm{.048} \pm  \bm{.001}$        & $    .240  \pm      .002         $& $    .157  \pm      .000         $& $\bm{.096} \pm  \bm{.000}$        & $    .069  \pm      .000\bullet$\\
		\hline
		$\L_h/$Square & $    .266  \pm      .018\bullet$& $    .248  \pm      .004\bullet$& $    .168  \pm      .018\bullet$& $    .061  \pm      .013     \bullet    $& $    .302  \pm      .012\bullet$& $    .222  \pm      .001\bullet$& $    .161  \pm      .002\bullet$& $    .110  \pm      .007\bullet$\\
		$\L_h/$Hinge & $    .268  \pm      .017\bullet$& $    .250  \pm      .002\bullet$& $    .161  \pm      .020\bullet$& $    .054  \pm      .008         $& $    .312  \pm      .009\bullet$& $    .205  \pm      .003\bullet$& $    .182  \pm      .007\bullet$& $    .091  \pm      .001\bullet$\\
		$\L_r/$Square & $    .267  \pm      .019\bullet$& $    .258  \pm      .003\bullet$& $    .191  \pm      .015\bullet$& $    .067  \pm      .012\bullet$& $    .302  \pm      .012\bullet$& $    .235  \pm      .006\bullet$& $    .174  \pm      .015\bullet$& $    .123  \pm      .007\bullet$\\
		$\L_r/$Hinge & $    .277  \pm      .014\bullet$& $    .253  \pm      .001\bullet$& $    .155  \pm      .010\bullet$& $    .061  \pm      .011   \bullet       $& $    .305  \pm      .010\bullet$& $    .224  \pm      .003\bullet$& $    .190  \pm      .005\bullet$& $    .117  \pm      .006\bullet$\\
		$\L_r/$Sigmoid & $    .261  \pm      .018\bullet$& $    .251  \pm      .001\bullet$& $    .164  \pm      .045         $& $    .059  \pm      .011         $& $    .304  \pm      .009\bullet$& $    .231  \pm      .008\bullet$& $    .134  \pm      .016\bullet$& $    .114  \pm      .003\bullet$\\
		\hline

		\multicolumn{8}{l}{Average precision (the greater, the better)}
		\\
		\hline
		$\tL_h/$Square & $    .714  \pm      .004\bullet$& $    .519  \pm      .014         $& $    .773  \pm      .021\bullet$& $    .905  \pm      .001         $& $    .640  \pm      .001         $& $\bm{.711} \pm  \bm{.002}$        & $    .811  \pm      .003\bullet$& $    .797  \pm      .004\bullet$\\
		$\tL_h/$Hinge & $    .696  \pm      .009\bullet$& $    .519  \pm      .015\bullet$& $    .798  \pm      .008\bullet$& $    .906  \pm      .001         $& $    .631  \pm      .001\bullet$& $    .691  \pm      .000\bullet$& $    .790  \pm      .003\bullet$& $\bm{.804} \pm  \bm{.004}$        \\
		\hline
		$\tL_r/$Square & $    .713  \pm      .007\bullet$& $    .518  \pm      .010         $& $    .777  \pm      .013\bullet$& $    .902  \pm      .002         $& $\bm{.640} \pm  \bm{.001}$        & $    .705  \pm      .002\bullet$& $    .827  \pm      .004\bullet$& $    .790  \pm      .005\bullet$\\
		$\tL_r/$Hinge & $    .710  \pm      .007\bullet$& $\bm{.525} \pm  \bm{.017}$        & $    .795  \pm      .015\bullet$& $    .903  \pm      .001         $& $    .636  \pm      .001\bullet$& $    .710  \pm      .001         $& $    .801  \pm      .002\bullet$& $    .798  \pm      .004\bullet$\\
		$\tL_r/$Sigmoid & $\bm{.720} \pm  \bm{.005}$        & $    .521  \pm      .018         $& $\bm{.809} \pm  \bm{.011}$        & $\bm{.906} \pm  \bm{.001}$        & $    .640  \pm      .000         $& $    .696  \pm      .000\bullet$& $\bm{.852} \pm  \bm{.003}$        & $    .765  \pm      .002\bullet$\\
		\hline
		$\L_h/$Square & $    .648  \pm      .022\bullet$& $    .457  \pm      .002\bullet$& $    .732  \pm      .026\bullet$& $    .846  \pm      .009 \bullet        $& $    .579  \pm      .010\bullet$& $    .651  \pm      .002\bullet$& $    .783  \pm      .018\bullet$& $    .708  \pm      .012\bullet$\\
		$\L_h/$Hinge & $    .645  \pm      .021\bullet$& $    .462  \pm      .003\bullet$& $    .744  \pm      .014\bullet$& $    .883  \pm      .012 \bullet         $& $    .539  \pm      .011\bullet$& $    .664  \pm      .003\bullet$& $    .754  \pm      .019\bullet$& $    .737  \pm      .010\bullet$\\
		$\L_r/$Square & $    .645  \pm      .024\bullet$& $    .451  \pm      .006\bullet$& $    .700  \pm      .024\bullet$& $    .802  \pm      .015\bullet$& $    .580  \pm      .008\bullet$& $    .619  \pm      .019\bullet$& $    .758  \pm      .005\bullet$& $    .685  \pm      .011\bullet$\\
		$\L_r/$Hinge & $    .635  \pm      .017\bullet$& $    .452  \pm      .004\bullet$& $    .751  \pm      .012\bullet$& $    .830  \pm      .013\bullet$& $    .573  \pm      .007\bullet$& $    .640  \pm      .003\bullet$& $    .758  \pm      .014\bullet$& $    .692  \pm      .013\bullet$\\
		$\L_r/$Sigmoid & $    .651  \pm      .022\bullet$& $    .453  \pm      .001\bullet$& $    .758  \pm      .025\bullet$& $    .843  \pm      .011 \bullet         $& $    .577  \pm      .006\bullet$& $    .647  \pm      .003\bullet$& $    .823  \pm      .015\bullet$& $    .696  \pm      .006\bullet$\\
		\hline
	\end{tabular}
\end{table*}

\section{Experiment}
The experiments for CCMN data are firstly reported, followed by the experiments for PML data.

\subsection{Study on CCMN Data}



We perform experiments on 8 benchmark data sets \footnote[1]{Publicly available at \url{http://mulan.sourceforge.net/datasets.html} and \url{http://meka.sourceforge.net/\#datasets}} which spanned a broad of applications: scene and mirflickr for image annotation, music\_emotion and music\_style for music recognition, yeast for protein function prediction as well as slashdot, enron and tmc for text categorization. For each data set, we randomly sample 50\% data for training and 30\% data for testing while the rest 20\% data is used for validation. To inject label noise into training examples, each class label is flipped according to noise rate $\rho_{-1}^{j}$ and $\rho_{+1}^{j}$ randomly sampled from $\{0.1,0.2,0.3,0.4,0.5\}$. We repeat experiments 5 times with different combinations of noise rates and report the averaging results.


We compare the performance of the proposed method with their baselines. Specifically, two variants of the proposed method, i.e., $\tLh$ (Eq.(\ref{eq:hl_ue})) and $\tLr$ (Eq.(\ref{eq:rl_ue}))  are compared with their baselines, i.e., $\Lh$ (Eq.(\ref{eq:hl})) and $\Lr$ (Eq.(\ref{eq:rl})). For $\tLh$ and $\Lh$, the surrogate loss function $\phi$ is defined as the least square loss and hinge loss; For $\tLr$ and $\Lr$, besides the least square loss and hinge loss, we also adopt Sigmoid loss, since it has been proved as a consistent surrogate loss with respective to ranking loss \cite{gao2013consistency}. For all methods, we use a linear model and a three-layer MLP with 128 hidden neural units as the based model. We use Adam \cite{KDP14} as the optimizer for 200 epochs and learning rate is chosen from \{5e-2,5e-3,5e-4\}. The $\ell_{2}$-regularization term is added with the parameter of 1e-4.  The testing performance of the model is reported with the best validation result out of all iterations. We conduct all experiments with Pytorch \cite{paszke2019pytorch}. 

As mentioned before, there are different criteria for evaluating the performance of multi-label learning. In our experiments, we employ three commonly used criteria including \textit{hamming loss}, \textit{ranking loss} and \textit{average precision}. More details can be found in \cite{zhang2013review}.






\begin{table*}[!htb]
	\centering
	\tiny
	\caption{Comparison results between $\tLh$, $\tLr$ and their baselines with diverse loss functions using MLP models,  where $\bullet$/$\circ$ indicates whether the proposed method is significantly superior/inferior to the comparing  methods via paired $t$-test (at 0.05 significance level).}\label{table:mlp_res}
	\renewcommand{\multirowsetup}{\centering}
	\begin{tabular}{l|l l l l l l l l}
		\hline
		Data& \makecell[c]{yeast} & \makecell[c]{enron} & \makecell[c]{scene} & \makecell[c]{slashdot} & \makecell[c]{music\_emotion} & \makecell[c]{music\_style} & \makecell[c]{mirflickr} & \makecell[c]{tmc} \\
		\hline
		\multicolumn{8}{l}{Hamming loss (the smaller, the better)}
		\\
		\hline
		$\tL_h/$Square & $\bm{.211} \pm  \bm{.002}$        & $    .056  \pm      .000         $& $    .124  \pm      .005         $& $    .022  \pm      .000         $& $    .193  \pm      .001         $& $\bm{.113} \pm  \bm{.001}$        & $    .109  \pm      .002\bullet$& $    .065  \pm      .000\bullet$\\
		$\tL_h/$Hinge & $    .211  \pm      .005         $& $    .057  \pm      .000         $& $    .124  \pm      .004         $& $    .023  \pm      .000         $& $    .193  \pm      .000         $& $    .113  \pm      .001\bullet$& $\bm{.106} \pm  \bm{.003}$        & $\bm{.064} \pm  \bm{.000}$        \\
		\hline
		$\tL_r/$Square & $    .214  \pm      .002\bullet$& $    .057  \pm      .000         $& $    .136  \pm      .005\bullet$& $    .022  \pm      .001\bullet$& $    .195  \pm      .003\bullet$& $    .115  \pm      .000\bullet$& $    .113  \pm      .003\bullet$& $    .069  \pm      .000\bullet$\\
		$\tL_r/$Hinge & $    .211  \pm      .004         $& $    .057  \pm      .001         $& $    .127  \pm      .005\bullet$& $    .023  \pm      .000         $& $\bm{.192} \pm  \bm{.001}$        & $    .116  \pm      .001\bullet$& $    .109  \pm      .003\bullet$& $    .065  \pm      .000\bullet$\\
		$\tL_r/$Sigmoid & $    .213  \pm      .004         $& $\bm{.056} \pm  \bm{.000}$        & $\bm{.123} \pm  \bm{.006}$        & $\bm{.021} \pm  \bm{.001}$        & $    .192  \pm      .000         $& $    .116  \pm      .000\bullet$& $    .109  \pm      .003         $& $    .068  \pm      .000\bullet$\\
		\hline
		$\L_h/$Square & $    .247  \pm      .015\bullet$& $    .071  \pm      .004\bullet$& $    .153  \pm      .020\bullet$& $    .025  \pm      .004         $& $    .206  \pm      .001\bullet$& $    .120  \pm      .001\bullet$& $    .132  \pm      .004\bullet$& $    .073  \pm      .000\bullet$\\
		$\L_h/$Hinge & $    .253  \pm      .017\bullet$& $    .058  \pm      .001         $& $    .146  \pm      .015\bullet$& $    .024  \pm      .002         $& $    .208  \pm      .003\bullet$& $    .120  \pm      .002\bullet$& $    .124  \pm      .007\bullet$& $    .075  \pm      .002\bullet$\\
		$\L_r/$Square & $    .250  \pm      .009\bullet$& $    .073  \pm      .005\bullet$& $    .140  \pm      .008\bullet$& $    .029  \pm      .003\bullet$& $    .203  \pm      .001\bullet$& $    .117  \pm      .001\bullet$& $    .136  \pm      .002\bullet$& $    .075  \pm      .002\bullet$\\
		$\L_r/$Hinge & $    .243  \pm      .012\bullet$& $    .061  \pm      .001\bullet$& $    .142  \pm      .006\bullet$& $    .027  \pm      .004         $& $    .200  \pm      .002\bullet$& $    .117  \pm      .002\bullet$& $    .135  \pm      .004\bullet$& $    .071  \pm      .002\bullet$\\
		$\L_r/$Sigmoid & $    .247  \pm      .018\bullet$& $    .067  \pm      .005\bullet$& $    .147  \pm      .012\bullet$& $    .027  \pm      .005         $& $    .208  \pm      .003\bullet$& $    .121  \pm      .002\bullet$& $    .126  \pm      .006\bullet$& $    .073  \pm      .002\bullet$\\
		\hline

		\multicolumn{8}{l}{Ranking loss (the smaller, the better)}
		\\
		\hline
		$\tL_h/$Square & $    .189  \pm      .006         $& $    .216  \pm      .009\bullet$& $    .122  \pm      .008\bullet$& $    .053  \pm      .004\bullet$& $    .231  \pm      .003         $& $    .153  \pm      .001         $& $    .081  \pm      .005         $& $\bm{.064} \pm  \bm{.001}$        \\
		$\tL_h/$Hinge & $    .200  \pm      .007\bullet$& $    .215  \pm      .007         $& $    .119  \pm      .006\bullet$& $    .048  \pm      .002         $& $    .242  \pm      .004\bullet$& $    .163  \pm      .001\bullet$& $    .086  \pm      .002\bullet$& $    .065  \pm      .001         $\\
		\hline
		$\tL_r/$Square & $.188 \pm  .008$        & $    .214  \pm      .007         $& $    .120  \pm      .009\bullet$& $    .050  \pm      .003         $& $    .233  \pm      .001\bullet$& $\bm{.152} \pm  \bm{.001}$        & $    .085  \pm      .006\bullet$& $    .065  \pm      .002         $\\
		$\tL_r/$Hinge & $    .190  \pm      .005         $& $    .214  \pm      .008\bullet$& $    .119  \pm      .009\bullet$& $    .047  \pm      .002         $& $    .233  \pm      .001\bullet$& $    .158  \pm      .002\bullet$& $    .081  \pm      .004\bullet$& $    .065  \pm      .002         $\\
		$\tL_r/$Sigmoid & $    \bm{.188}  \pm      \bm{.007}         $& $\bm{.206} \pm  \bm{.010}$        & $\bm{.108} \pm  \bm{.006}$        & $\bm{.044} \pm  \bm{.001}$        & $\bm{.230} \pm  \bm{.001}$        & $    .153  \pm      .003         $& $\bm{.077} \pm  \bm{.004}$        & $    .078  \pm      .003\bullet$\\
		\hline
		$\L_h/$Square & $    .228  \pm      .018\bullet$& $    .252  \pm      .001\bullet$& $    .150  \pm      .016\bullet$& $    .052  \pm      .005\bullet$& $    .293  \pm      .006\bullet$& $    .207  \pm      .001\bullet$& $    .114  \pm      .007\bullet$& $    .115  \pm      .004\bullet$\\
		$\L_h/$Hinge & $    .238  \pm      .010\bullet$& $    .255  \pm      .001\bullet$& $    .150  \pm      .026\bullet$& $    .049  \pm      .003\bullet$& $    .291  \pm      .007\bullet$& $    .194  \pm      .006\bullet$& $    .095  \pm      .013         $& $    .095  \pm      .001\bullet$\\
		$\L_r/$MSE & $    .233  \pm      .016\bullet$& $    .255  \pm      .003\bullet$& $    .155  \pm      .018\bullet$& $    .052  \pm      .005\bullet$& $    .294  \pm      .007\bullet$& $    .203  \pm      .003\bullet$& $    .117  \pm      .004\bullet$& $    .121  \pm      .006\bullet$\\
		$\L_r/$Hinge & $    .231  \pm      .017\bullet$& $    .249  \pm      .002\bullet$& $    .146  \pm      .017\bullet$& $    .054  \pm      .008\bullet$& $    .294  \pm      .004\bullet$& $    .203  \pm      .003\bullet$& $    .112  \pm      .004\bullet$& $    .115  \pm      .007\bullet$\\
		$\L_r/$Sigmoid & $    .233  \pm      .016\bullet$& $    .252  \pm      .003\bullet$& $    .180  \pm      .056         $& $    .049  \pm      .004         $& $    .290  \pm      .009\bullet$& $    .208  \pm      .002\bullet$& $    .112  \pm      .010\bullet$& $    .113  \pm      .004\bullet$\\
		\hline

		\multicolumn{8}{l}{Average precision (the greater, the better)}
		\\
		\hline
		$\tL_h/$Square & $\bm{.736} \pm  \bm{.004}$        & $    .503  \pm      .010\bullet$& $    .797  \pm      .015\bullet$& $    .903  \pm      .002\bullet$& $    .649  \pm      .002         $& $\bm{.719} \pm  \bm{.004}$        & $    .874  \pm      .006\bullet$& $    .786  \pm      .002\bullet$\\
		$\tL_h/$Hinge & $    .730  \pm      .003         $& $    .493  \pm      .009\bullet$& $    .801  \pm      .010\bullet$& $    .906  \pm      .001\bullet$& $    .638  \pm      .002\bullet$& $    .714  \pm      .005         $& $    .874  \pm      .002\bullet$& $\bm{.790} \pm  \bm{.001}$        \\
		\hline
		$\tL_r/$Square & $    .735  \pm      .004         $& $\bm{.516} \pm  \bm{.013}$        & $    .792  \pm      .013\bullet$& $    .905  \pm      .002         $& $    .646  \pm      .002         $& $    .717  \pm      .002         $& $    .869  \pm      .007\bullet$& $    .781  \pm      .003\bullet$\\
		$\tL_r/$Hinge & $    .734  \pm      .005         $& $    .497  \pm      .010\bullet$& $    .797  \pm      .017\bullet$& $    .907  \pm      .002         $& $    .647  \pm      .002\bullet$& $    .712  \pm      .003\bullet$& $    .877  \pm      .004\bullet$& $    .788  \pm      .001\bullet$\\
		$\tL_r/$Sigmoid & $    .734  \pm      .002         $& $    .515  \pm      .015         $& $\bm{.816} \pm  \bm{.009}$        & $\bm{.910} \pm  \bm{.001}$        & $\bm{.650} \pm  \bm{.002}$        & $    .710  \pm      .003         $& $\bm{.882} \pm  \bm{.006}$        & $    .758  \pm      .002\bullet$\\
		\hline
		$\L_h/$Square & $    .685  \pm      .016\bullet$& $    .448  \pm      .010\bullet$& $    .770  \pm      .016\bullet$& $    .892  \pm      .007 \bullet        $& $    .591  \pm      .003\bullet$& $    .672  \pm      .007\bullet$& $    .840  \pm      .005\bullet$& $    .707  \pm      .015\bullet$\\
		$\L_h/$Hinge & $    .674  \pm      .014\bullet$& $    .451  \pm      .005\bullet$& $    .763  \pm      .020\bullet$& $    .896  \pm      .013  \bullet       $& $    .580  \pm      .006\bullet$& $    .673  \pm      .001\bullet$& $    .863  \pm      .027         $& $    .710  \pm      .017\bullet$\\
		$\L_r/$Square & $    .683  \pm      .015\bullet$& $    .452  \pm      .007\bullet$& $    .761  \pm      .021\bullet$& $    .892  \pm      .012 \bullet         $& $    .588  \pm      .004\bullet$& $    .679  \pm      .003\bullet$& $    .837  \pm      .007\bullet$& $    .702  \pm      .015\bullet$\\
		$\L_r/$Hinge & $    .689  \pm      .017\bullet$& $    .462  \pm      .004\bullet$& $    .771  \pm      .013\bullet$& $    .877  \pm      .013 \bullet        $& $    .591  \pm      .004\bullet$& $    .677  \pm      .002\bullet$& $    .842  \pm      .006\bullet$& $    .712  \pm      .016\bullet$\\
		$\L_r/$Sigmoid & $    .685  \pm      .015\bullet$& $    .461  \pm      .004\bullet$& $    .752  \pm      .037\bullet$& $    .902  \pm      .003\bullet$& $    .593  \pm      .007\bullet$& $    .663  \pm      .005\bullet$& $    .850  \pm      .010\bullet$& $    .700  \pm      .012\bullet$\\
		\hline
	\end{tabular}
\end{table*}

Table \ref{table:linear_res} and Table \ref{table:mlp_res} report comparison results of the proposed methods between baseline methods using linear and MLP models, respectively. For each data set, paired $t$-test based on 5 repeats (at 0.05 significance level) is conducted to show weather the proposed unbiased estimator is significantly different from the comparing methods. From the results, it is obvious that the proposed method outperforms their corresponding baselines with significant superiority in almost all cases, which validates the effectiveness of the proposed method. The performance of $\tL_r$ is generally superior to $\tL_h$, which validates the label correlation is beneficial for learning with class-conditional multi-label noise.

\subsection{Study on PML Data}

To validate the effectiveness of the proposed uPML method, we perform experiments on totally 8 data sets, including 3 real-world PML datasets, music\_style, music\_emotion and mirflickr as well as 5 multi-label dataets, yeast, enron, scene, slashdot and tmc. Regarding the multi-label datasets, for each instance, the negative label is flipped according to noise rate $\rho$ randomly sampled from $\{0.1,0.2,0.3,0.4,0.5,0.6\}$. We repeat experiments 5 times with different noise rate and report the averaging results.

We compare with five state-of-art PML algorithms: PML-NI\cite{xie2020pmlni}, PARTICLE (including two implementations: PAR-VLS and PAR-MAP) \cite{zhang2020pml}, PML-LRS \cite{LS19} and fPML \cite{yu2018pml}. To make a fair comparison, we use a linear classifier as the base model for uPML. We use Adam \cite{KDP14} as the optimizer for 200 epochs and the learning rate is chosen from \{0.05,0.005,0.0005\}. The $\ell_{2}$-regularization term is added with the parameter of 1e-4. The testing performance of the model is reported with the best validation result out of all iterations. For the other comparing methods, parameters are determined by the performance on validation set if no default value given in their literature.


\begin{table*}[!htb]
	\tiny
	\centering
	\caption{Comparison results between the proposed method (using linear model) and PML methods,  where $\bullet$/$\circ$ indicates whether the proposed method is significantly superior/inferior to the comparing  method via paired $t$-test (at 0.05 significance level).}\label{table:pml_res}
	\renewcommand{\multirowsetup}{\centering}
	\begin{tabular}{l|l l l | l l l l l}
		\hline
		Data& \makecell[c]{music\_style} & \makecell[c]{music\_emotion} & \makecell[c]{mirflickr} & \makecell[c]{yeast} & \makecell[c]{enron} & \makecell[c]{scene} & \makecell[c]{slashdot} & \makecell[c]{tmc}\\
		\hline
		\multicolumn{8}{l}{Hamming loss (the smaller, the better)}
		\\
		\hline
		$\tL_h^{\text{PML}}/$Square & $\bm{.117} \pm  \bm{.001}       $ & $    .196  \pm      .002         $& $    .170  \pm      .003\bullet $ & $    .210  \pm      .001         $& $\bm{.056} \pm  \bm{.000}       $ & $\bm{.124} \pm  \bm{.000}       $ & $    .018  \pm      .000\bullet $ & $\bm{.064} \pm  \bm{.000}       $ \\
		$\tL_r^{\text{PML}}/$Sigmoid & $    .121  \pm      .003         $& $\bm{.195} \pm  \bm{.001}       $ & $    \bm{.165}  \pm      \bm{.002}         $& $\bm{.208} \pm  \bm{.001}       $ & $    .064  \pm      .010         $& $    .145  \pm      .001\bullet $ & $\bm{.017} \pm  \bm{.000}       $ & $    .066  \pm      .000\bullet $ \\
		\hline
		PMLNI & $    .157  \pm      .008\bullet $ & $    .253  \pm      .005\bullet $ & $    .228  \pm      .005\bullet $ & $    .269  \pm      .027\bullet $ & $    .128  \pm      .004\bullet $ & $    .322  \pm      .047\bullet $ & $    .065  \pm      .007\bullet $ & $    .082  \pm      .002\bullet $ \\
		PMLLRS & $    .167  \pm      .020\bullet $ & $    .361  \pm      .049\bullet $ & $    .209  \pm      .007\bullet $ & $    .244  \pm      .016\bullet $ & $    .174  \pm      .006\bullet $ & $    .155  \pm      .021\bullet $ & $    .023  \pm      .002\bullet $ & $    .088  \pm      .003\bullet $ \\
		PARVLS & $    .124  \pm      .001\bullet $ & $    .254  \pm      .002\bullet $ & $    .219  \pm      .005\bullet $ & $    .259  \pm      .009\bullet $ & $    .107  \pm      .006\bullet $ & $    .248  \pm      .017\bullet $ & $    .062  \pm      .004\bullet $ & $    .084  \pm      .001\bullet $ \\
		PARMAP & $    .142  \pm      .004\bullet $ & $    .252  \pm      .002\bullet $ & $.172 \pm  .002       $ & $    .239  \pm      .013\bullet $ & $    .157  \pm      .008\bullet $ & $    .289  \pm      .021\bullet $ & $    .019  \pm      .001\bullet $ & $    .101  \pm      .009\bullet $ \\
		fPML & $    .143  \pm      .000\bullet $ & $    .221  \pm      .000\bullet $ & $    .254  \pm      .001\bullet $ & $    .303  \pm      .002\bullet $ & $    .111  \pm      .008\bullet $ & $    .179  \pm      .000\bullet $ & $    .054  \pm      .000\bullet $ & $    .101  \pm      .000\bullet $ \\
		\hline

		\multicolumn{8}{l}{Ranking loss (the smaller, the better)}
		\\
		\hline
		$\tL_h^{\text{PML}}/$Square & $    .154  \pm      .004 \bullet        $& $\bm{.225} \pm  \bm{.003}       $ & $    .100  \pm      .002\bullet $ & $    .199  \pm      .005\bullet $ & $    .193  \pm      .006\bullet $ & $    .120  \pm      .008\bullet         $& $    .056  \pm      .003\bullet $ & $\bm{.062} \pm  \bm{.003}       $ \\
		$\tL_r^{\text{PML}}/$Sigmoid & $\bm{.146} \pm  \bm{.004}       $ & $    .229  \pm      .000\bullet $ & $\bm{.090} \pm  \bm{.002}       $ & $\bm{.189} \pm  \bm{.001}       $ & $\bm{.172} \pm  \bm{.003}       $ & $\bm{.112} \pm  \bm{.002}       $ & $\bm{.054} \pm  \bm{.002}       $ & $    .066  \pm      .000\bullet $ \\
		\hline
		PMLNI & $    .149  \pm      .003         $& $    .250  \pm      .003\bullet $ & $    .121  \pm      .005\bullet $ & $    .231  \pm      .016\bullet $ & $    .338  \pm      .006\bullet $ & $    .207  \pm      .018\bullet $ & $    .118  \pm      .006\bullet $ & $    .123  \pm      .007\bullet $ \\
		PMLLRS & $    .154  \pm      .005\bullet      $& $    .265  \pm      .002\bullet $ & $    .112  \pm      .003\bullet $ & $    .223  \pm      .009\bullet $ & $    .332  \pm      .013\bullet $ & $    .198  \pm      .040\bullet $ & $    .086  \pm      .002\bullet $ & $    .122  \pm      .010\bullet $ \\
		PARVLS & $    .238  \pm      .004\bullet $ & $    .349  \pm      .005\bullet $ & $    .230  \pm      .032\bullet $ & $    .245  \pm      .007\bullet $ & $    .448  \pm      .007\bullet $ & $    .388  \pm      .016\bullet $ & $    .216  \pm      .058\bullet $ & $    .205  \pm      .005\bullet $ \\
		PARMAP & $    .207  \pm      .003\bullet $ & $    .310  \pm      .003\bullet $ & $    .110  \pm      .009\bullet $ & $    .245  \pm      .017\bullet $ & $    .293  \pm      .016\bullet $ & $    .318  \pm      .023\bullet $ & $    .064  \pm      .004\bullet $ & $    .203  \pm      .009\bullet $ \\
		fPML & $    .155  \pm      .003 \bullet        $& $    .264  \pm      .007\bullet $ & $    .107  \pm      .010\bullet $ & $    .219  \pm      .008\bullet $ & $    .342  \pm      .031\bullet $ & $    .230  \pm      .056\bullet $ & $    .093  \pm      .010\bullet $ & $    .132  \pm      .004\bullet $ \\
		\hline

		\multicolumn{8}{l}{Average precision (the greater, the better)}
		\\
		\hline
		$\tL_h^{\text{PML}}/$MSE & $    .716  \pm      .004         $& $\bm{.649} \pm  \bm{.004}       $ & $    .831  \pm      .002\bullet $ & $    .730  \pm      .004         $& $    .550  \pm      .005\bullet $ & $    .799  \pm      .010         $& $    .898  \pm      .003\bullet $ & $\bm{.799} \pm  \bm{.005}       $ \\
		$\tL_r^{\text{PML}}/$Sigmoid & $    .708  \pm      .009\bullet $ & $    .641  \pm      .002\bullet $ & $    \bm{.846}  \pm      \bm{.004}         $& $\bm{.732} \pm  \bm{.004}       $ & $\bm{.563} \pm  \bm{.004}       $ & $\bm{.802} \pm  \bm{.003}       $ & $\bm{.899} \pm  \bm{.003}       $ & $    .785  \pm      .003\bullet $ \\
		\hline
		PMLNI & $\bm{.723} \pm  \bm{.005}       $ & $    .602  \pm      .002\bullet $ & $    .792  \pm      .007\bullet $ & $    .687  \pm      .018\bullet $ & $    .270  \pm      .013\bullet $ & $    .695  \pm      .021\bullet $ & $    .614  \pm      .036\bullet $ & $    .689  \pm      .010\bullet $ \\
		PMLLRS & $    .706  \pm      .006\bullet $ & $    .580  \pm      .002\bullet $ & $    .814  \pm      .005\bullet $ & $    .704  \pm      .010\bullet $ & $    .269  \pm      .017\bullet $ & $    .715  \pm      .043\bullet $ & $    .829  \pm      .015\bullet $ & $    .692  \pm      .017\bullet $ \\
		PARVLS & $    .662  \pm      .004\bullet $ & $    .522  \pm      .004\bullet $ & $    .697  \pm      .063\bullet $ & $    .692  \pm      .010\bullet $ & $    .225  \pm      .008\bullet $ & $    .563  \pm      .010\bullet $ & $    .340  \pm      .158\bullet $ & $    .618  \pm      .006\bullet $ \\
		PARMAP & $    .659  \pm      .006\bullet $ & $    .541  \pm      .003\bullet $ & $.842 \pm  .004       $ & $    .680  \pm      .015\bullet $ & $    .304  \pm      .013\bullet $ & $    .567  \pm      .023\bullet $ & $    .862  \pm      .010\bullet $ & $    .595  \pm      .015\bullet $ \\
		fPML & $    .704  \pm      .003\bullet $ & $    .585  \pm      .004\bullet $ & $    .820  \pm      .018\bullet $ & $    .709  \pm      .008\bullet $ & $    .299  \pm      .033\bullet $ & $    .676  \pm      .064\bullet $ & $    .851  \pm      .011\bullet $ & $    .698  \pm      .011\bullet $ \\
		\hline
		
	\end{tabular}
\end{table*}

Table \ref{table:pml_res} reports the comparison result of two variants of the proposed uPML method against comparing PML methods. For each data set, pair $t$-test based on five repeats (at 0.05 significance level) is conducted to show weather the proposed unbiased estimator is significantly different from the comparing methods. The proposed uPML method significantly outperforms the comparing PML methods in almost all cases except for the case on music\_style, where PMLNI is significantly  better than uPML in terms of \emph{average precision}. In general, uPML with $\tL_r$ is better than uPML with $\tL_h$ due to the label correlation is considered for the former one. Theses results convincingly demonstrate the effectiveness of the proposed method.

\section{Conclusion}

In this paper, we study the problem of multi-label classification with class-conditional multiple noisy labels, where multiple class labels assigned to each instance may be corrupted simultaneously with class-conditional probabilities. From the perspective of the unbiased estimator, we derive efficient methods for solving CCMN problems with theoretical guarantee. Generally, we prove that learning from class-conditional multiple noisy labels with the proposed unbiased estimators is consistent with respective to hamming loss and ranking loss. Furthermore, we propose a novel method called uPML for solving PML problems, which can be regarded as a special case of CCMN framework. Empirical studies on multiple data sets validate the effectiveness of the proposed method. In the future, we will study CCMN with more loss functions.

\bibliographystyle{icml2020}\small
\bibliography{ref}

\newpage
\appendix

\section{Proof for Lemma 1}

\begin{proof}
	By using $\rho_{j}$ and $\rho_{-j}$ to denote $\rho_{y_{j}}^{j}$ and $\rho_{-y_{j}}^{j}$, respectively, we have
	\begin{align*}
	&\quad\E_{\ty}\left[\tLh(\f,\ty)\right] \\
	&=\sum_{j=1}^{q}(1-\rho_{j})\tphi(f_{j},y_{j})+\rho_{j}\tphi(f_{j},-y_{j}) \\
	&=\sum_{j=1}^{q}\frac{1}{1-\rho_{-1}^j-\rho_{+1}^j}[(1-\rho_{j})[(1-\rho_{-j})\phi(f_{j},y_{j})-\rho_{j}\phi(f_{j},-y_{j})]+\rho_{j}[(1-\rho_{j})\phi(f_{j},-y_{j})-\rho_{-j}\phi(f_{j},y_{j})]]\\
	&=\sum_{j=1}^{q}\frac{1}{1-\rho_{-1}^j-\rho_{+1}^j}[(1-\rho_{j})(1-\rho_{-j})\phi(f_{j},y_{j})-\rho_{j}\rho_{-j}\phi(f_{j},y_{j})]\\
	&=\sum_{j=1}^{K}\frac{1}{1-\rho_{-1}^j-\rho_{+1}^j}(1-\rho_{j}-\rho_{-j})\phi(f_{j},y_{j})=\Lh(\f,\y)
	\end{align*}
	which completes proof.
\end{proof}

\section{Proof for Theorem 2}

\begin{proof}
	The proof is mainly composed of the following two lemmas.
	
	\begin{lemma}\label{lemma:hl_rc}
		Let $\R_{n}(\tLh\circ\F)$ be the Rademacher complexity with respective to the unbiased estimator $\tLh$ and function class $\F$ over $S_{\rho}$ of $n$ training points drawn from $\D_{\rho}$, which can be defined as
		\begin{equation*}
		\R_{n}(\tLh\circ\F)=\E_{S_{\rho}}\E_{\sigma}\left[\sup_{\f\in{\F}}\frac{1}{n}\sum_{i=1}^{n}\sigma_{i}\tLh(\f(\x_{i}),\ty_{i})\right]
		\end{equation*}  
		Then, 
		\begin{equation*}
		\R_{n}(\tLh\circ\F)\leq qK_{\rho}\R_{n}(\F)
		\end{equation*}  
		where $K_{\rho}$ is the Lipschitz constant of $\tLh$.
	\end{lemma}
\begin{proof}
	Recall that  $\tLh(\f(\x),\y)=\sum_{j=1}^{q}\tphi(y_{j}f_{j}(\x))$, then, we have
	\begin{align*}
	&\quad\R_{n}(\tLh\circ\F)\\
	&=\E_{S_{\rho}}\E_{\sigma}\left[\sup_{f_{1},...,f_{q}\in\F}\frac{1}{n}\sum_{i=1}^{n}\sigma_{i}\sum_{j=1}^{q}\tphi(y_{j}^{(i)}f_{j}(\x_{i}))\right]\\
	&=\E_{\mathcal{X}}\E_{\sigma}\left[\sup_{f_{1},...,f_{q}\in\F}\frac{1}{n}\sum_{\x_{i}\in\mathcal{X}}\sigma_{i}\sum_{j=1}^{q}\tphi(y_{j}^{(i)}f_{j}(\x_{i}))\right]\\
	&\leq \sum_{j=1}^{q}\E_{\mathcal{X}}\E_{\sigma}\left[\sup_{f_{j}\in\F}\frac{1}{n}\sum_{\x_{i}\in\mathcal{X}}\sigma_{i}\tphi(y^{(i)}_{j}f_{j}(\x_{i}))\right]\\
	&=q\R_{n}(\tphi\circ\F)
	\end{align*}
	Sequentially, according to Talagrand's contraction lemma \cite{LM13}, we have
	\begin{align*}
	\R_{n}(\tLh\circ\F)&\leq  q\R_{n}(\tphi\circ\F)\\ 
	&\leq qK_{\rho}\R_{n}(\F)
	\end{align*}
	which completes the proof.
\end{proof}

	Without loss of generality, assume that $\mu=\max_{j}\frac{1}{1-\rho_{-1}^j-\rho_{+1}^j},\forall j\in[q]$.
	\begin{lemma}\label{lemma:hl_rcd}
		For any $\delta>0$, with probability at least $1-\delta$,
		\begin{equation*}
		\max_{\f\in\mathcal{F}}\left\vert\hR_{\tL_h}(\f)-R_{\tL_h}(\f)\right\vert\leq 2qK_{\rho}\R_{n}(\F)+q\mu\Theta\sqrt{\frac{\ln\frac{2}{\delta}}{2n}}
		\end{equation*}
	\end{lemma}
\begin{proof}
	Since both two directions can be proved in the same way, we consider one single direction $\sup_{f_{1},...,f_{q}\in\F}(\hR_{\tL}(\f)-R_{\tL}(\f))$. Note that the change in $\x_{i}$ leads to a perturbation of at most  $\frac{q\mu\Theta}{n}$ by replacing a single point $(\x_{i},\ty_{i})$ with $(\x_{i}^{\prime},\ty_{i}^{\prime})$, since the change in any $\tilde{y}^{(i)}_{j}$ leads to a perturbation as:
	\begin{equation}
	\frac{1}{n}\left|\frac{(1-\rho_{-1}^j)\phi(f_j)-\rho_{+1}^j\phi(-f_j)-[(1-\rho_{+1}^j)\phi(-f_j)-\rho_{-1}^j\phi(f_j)]}{1-\rho_{-1}^j-\rho_{+1}^j}\right|
	=\frac{1}{n}\left|\frac{\phi(f_j)-\phi(-f_j)}{1-\rho_{-1}^j-\rho_{+1}^{j}}\right|\leq\frac{\mu\Theta}{n}
	\end{equation}
	By using McDiarmid's inequality \cite{mohri2018foundations} to the single-direction uniform deviation $\sup_{f_{1},...,f_{k}\in\F}(\hR_{\tL}(\f)-R_{\tL}(\f))$, we have
	\begin{equation*}
	\mathbb{P}\left\{\sup_{\f\in\F}(\hR_{\tL}(\f)-R_{\tL}(\f))-\E\left[\sup_{\f\in\F}(\hR_{\tL}(\f)-R_{\tL}(\f))\right]\geq\epsilon\right\}\leq\exp\left(-\frac{2\epsilon^{2}}{n(\frac{q\mu\Theta}{n})^{2}}\right)
	\end{equation*}
	or equivalently, with probability at least $1-\delta$,
	\begin{equation*}
	\sup_{\f\in\F}(\hR_{\tL}(\f)-R_{\tL}(\f))\leq\E\left[\sup_{\f\in\F}(\hR_{\tL}(\f)-R_{\tL}(\f))\right]+q\mu\Theta\sqrt{\frac{\ln\frac{1}{\delta}}{2n}}
	\end{equation*}
	According to \cite{CE09}, it is straightforward to show that
	\begin{equation*}
	\E\left[\sup_{\f\in\F}(\hR_{\tL}(\f)-R_{\tL}(\f))\right]\leq 2\R_{n}(\tL\circ\F)
	\end{equation*}
	With the lemma \ref{lemma:hl_rc}, we complete the proof.
\end{proof}

	\noindent\textit{Proof of Theorem 2}. Based on the Lemma \ref{lemma:hl_rcd}, with $\f^{*}=\argmin_{\f\in\mathcal{F}}R_{\L,\D}(\f)$, it is obvious to prove the generalization error bound as follows:
	\begin{align*}
	&\quad R_{\L,\D}(\hat\f)-R_{\L,\D}(\f^{*})\\
	&=R_{\tL,\D_{\rho}}(\hat\f)-R_{\tL,\D_{\rho}}(\f^{*})\\
	&=\left(\hR_{\tL}(\hat\f)-\hR_{\tL}(\f^{*})\right) + \left(R_{\tL,\D_{\rho}}(\hat\f)-\hR_{\tL}(\hat\f) \right) +\left(\hR_{\tL}(\f^{*})-R_{\tL,\D_{\rho}}(\f^{*})\right)\\
	&\leq 0+2\max_{\f\in\F}\left\vert \hR_{\tL}(\f)-R_{\tL,\D_{\rho}}(\f)\right\vert
	\end{align*}
	The fist equation holds due to the unbiasedness of the estimator and for the last line, we used the fact $\hR_{\tL}(\hat{\f})\leq\hR_{\tL}(\f^{*})$ by the definition of $\hat\f$.
\end{proof}

\section{Proof for Theorem 3}

\begin{proof}
	With respective to $\tL_h$, the conditional surrogate loss
	can be defined as
	\begin{align*}
	\widetilde{W}(\p,\f)&=\sum_{\y\in\mathcal{Y}}p_{\y}\E_{\ty|\y}[\tL(\f(\x),\ty)]=\sum_{j=1}^{q}\sum_{\y\in\mathcal{Y}}p_{\y}\phi(y_jf_j(\x))\\
	&= \sum_{j=1}^q p_j^{+}\phi(f_j(\x))+ p_j^{-}\phi(-f_j(\x))
	\end{align*}
	where $p_j^{+}=\sum_{\y:y_j=+1}p_{\y}$ and $p_j^{-}=\sum_{\y:y_j=-1}p_{\y}$. This yields minimizing $\widetilde{W}(\p,\f)$ is equivalent to minimizing $W(\p,\f)$.
	Accordingly, we have 
	\begin{equation*}
	\widetilde{W}^{*}(\p)=W^*(\p)=\inf_{\f}W(\p,\f)=\sum_{j=1}^q\inf_{f_j}W_{j}(p_j^{+},f_j)
	\end{equation*}
	where $W_{j}(p_j^{+},f_j)=p_j^{+}\phi(f_j)+p_j^{-}\phi(-f_j)$. Based on the consistency of binary classification setting proposed in \cite{bartlett2006convexity,zhang2004consistency}, it is easy to prove that $\tL_h$ is consistent to hamming loss. Therefore, based on Corollary 25 in \cite{zhang2004consistency}, we have
	\begin{equation*}
	R_{L_h}(\hat\f)-R_{L_h}^{*}\leq\xi(R_{\L_{h}}(\hat\f)-R_{\L_{h}}^{*})
	\end{equation*}
	where $\xi$ is a non-negative concave function with $\xi(0)=0$. 
\end{proof}

\section{Proof for Lemma 2}

\begin{proof}
	With respective to each label pair $(y_{j},y_{k})$, considering the cases $(y_j=+1, y_k=-1)$, $(y_{j}=-1, y_{k}=+1)$, $(y_{j}=+1, y_{k}=+1)$ and $(y_{j}=-1, y_{k}=-1)$  separately, we can obtain the modified loss function $\tphi(j,k)$ by solving the following linear equation:
	
	\begin{equation*}
	\left[
	\begin{array}{cccc}
	\alpha_{+1}^j\alpha_{-1}^k&\rho_{+1}^j\rho_{-1}^k  & \rho_{+1}^j\alpha_{-1}^k    & \rho_{-1}^k\alpha_{+1}^j\\
	\rho_{-1}^j\rho_{+1}^k        & \alpha_{-1}^j\alpha_{+1}^k& \rho_{-1}^j\alpha_{+1}^k  & \rho_{-1}^j\alpha_{+1}^k\\
	\alpha_{+1}^j\rho_{+1}^k& \rho_{+1}^j\alpha_{+1}^k & \rho_{+1}^j\rho_{+1}^k  & \alpha_{+1}^j\alpha_{+1}^k \\
	\rho_{-1}^j\alpha_{-1}^k&\alpha_{-1}^j\rho_{-1}^k  & \alpha_{-1}^j\alpha_{-1}^k& \rho_{-1}^j\rho_{-1}^k 
	\end{array}
	\right ]
	\left[
	\begin{array}{cccc}
	\tphi((f_j,f_k),(+1,-1))\\
	\tphi((f_j,f_k),(-1,+1))\\
	\tphi((f_j,f_k),(+1,+1))\\
	\tphi((f_j,f_k),(-1,-1))
	\end{array}
	\right ]
	=
	\left[
	\begin{array}{cccc}
	\phi(f_j-f_k)\\
	\phi(f_k-f_j)\\
	\phi(0) \\
	\phi(0) 
	\end{array}
	\right ]
	\end{equation*}
	where $\alpha_{+1}^j=1-\rho_{+1}^j$. Then, the solution can be obtained:
	\begin{align*}
	&\tphi((f_j,f_k),(+1,-1))=\kappa_{jk}\left[(1-\rho_{-1}^{j})(1-\rho_{+1}^{k})\phi(f_{jk})+\rho_{+1}^{j}\rho_{-1}^{k}\phi(-f_{jk})\right]\\
	&\tphi((f_j,f_k),(-1,+1))=\kappa_{jk}\left[(1-\rho_{+1}^{j})(1-\rho_{-1}^{k})\phi(-f_{jk})+\rho_{-1}^{j}\rho_{+1}^{k}\phi(f_{jk})\right]\\
	&\tphi((f_j,f_k),(+1,+1))=-\kappa_{jk}\left[\rho_{+1}^{j}(1-\rho_{-1}^{k})\phi(-f_{jk})+\rho_{+1}^{k}(1-\rho_{-1}^{j})\phi(f_{jk})\right]\\
	&\tphi((f_j,f_k),(-1,-1))=-\kappa_{jk}\left[\rho_{-1}^{j}(1-\rho_{+1}^{k})\phi(f_{jk})+\rho_{-1}^{k}(1-\rho_{+1}^{j})\phi(-f_{jk})\right]
	\end{align*}
	which completes the proof.
\end{proof}

\section{Proof for Theorem 4}

\begin{proof}
	The proof is mainly composed of the following two lemmas.
	
	\begin{lemma}\label{lemma:rl_rc}
		Let $\R_{n}(\tLr\circ\F)$ be the Rademacher complexity with respective to the unbiased estimator $\tLr$ and function class $\F$ over $S_{\rho}$ of $n$ training points drawn from $\D_{\rho}$, which can be defined as
		\begin{equation*}
		\R_{n}(\tLr\circ\F)=\E_{S_{\rho}}\E_{\sigma}\left[\sup_{\f\in{\F}}\frac{1}{n}\sum_{i=1}^{n}\sigma_{i}\tLr(\f(\x_{i}),\ty_{i})\right]
		\end{equation*}  
		Then, 
		\begin{equation*}
		\R_{n}(\tLr\circ\F)\leq 2q(q-1)K_{\rho}\R_{n}(\F)
		\end{equation*}
		where $K_{\rho}$ is the Lipschitz constant of $\tLr$.  
	\end{lemma}
	
	\begin{proof}
		Recall that $\tLr(\f(\x),\ty)=\sum_{1\leq j<k\leq q}\tphi((f_j,f_k),(y_j,y_k))$, then, we have
		\begin{align}
		\R_{n}(\tLr\circ\F)&=\E_{S_{\rho}}\E_{\sigma}\left[\sup_{f_{1},...,f_{q}\in\F}\frac{1}{n}\sum_{i=1}^{n}\sigma_{i}\sum_{1\leq j<k\leq q}\tphi((f_j,f_k),(y_j,y_k))\right]\nonumber\\
		&=\E_{\mathcal{X}}\E_{\sigma}\left[\sup_{f_{1},...,f_{q}\in\F}\frac{1}{n}\sum_{\x_{i}\in\mathcal{X}}\sigma_{i}\sum_{1\leq j<k\leq q}\tphi((f_j,f_k),(y_j,y_k))\right]\nonumber\\
		&\leq  \sum_{1\leq j<k\leq q}\E_{\mathcal{X}}\E_{\sigma}\left[\sup_{f_{1},...,f_{q}\in\F}\frac{1}{n}\sum_{\x_{i}\in\mathcal{X}}\sigma_{i}\tphi((f_j,f_k),(y_j,y_k))\right]\label{eq:rl_rc}
		\end{align}
		Sequentially, let $(y,y')$ be the current label pair to be cumulated, then, according to Talagrand's contraction lemma \cite{LM13}, we have
		\begin{align*}
		&\E_{\mathcal{X}}\E_{\sigma}\left[\sup_{f_{y},f_{y'}\in\F}\frac{1}{n}\sum_{\x_{i}\in\mathcal{X}}\sigma_{i}\tphi\left((f_y,f_{y'}),(y,y')\right)\right]\\
		&\leq K_{\rho}\E_{\mathcal{X}}\E_{\sigma}\left[\sup_{f_{y},f_{y'}\in\F}\frac{1}{n}\sum_{\x_{i}\in\mathcal{X}}\sigma_{i}(f_{y}(\x_{i})-f_{y'}(\x_{i}))\right]\\
		&\leq K_{\rho}\E_{\mathcal{X}}\E_{\sigma}\left[\sup_{f_{y}\in\F}\frac{1}{n}\sum_{\x_{i}\in\mathcal{X}}\sigma_{i}f_{y}(\x_{i})\right]+K_{\rho}\E_{\mathcal{X}}\E_{\sigma}\left[\sup_{f_{y'}\in\F}\frac{1}{n}\sum_{\x_{i}\in\mathcal{X}}\sigma_{i}f_{y'}(\x_{i})\right]\\
		&=2K_{\rho}\R_{n}(\F)
		\end{align*}
		Where $f_{y}$ represent the classifier corresponding class label $y$. Then, by combining with \ref{eq:rl_rc}, it is easy to prove that $\R_{n}(\tLr\circ\F)\leq 2q(q-1)K_{\rho}\R_{n}(\F)$.  
	\end{proof}

	Without loss of generality, assume that $\mu=\max_j\frac{1-\min(\rho_{-1}^j-\rho_{+1}^j,\rho_{+1}^j-\rho_{-1}^j)}{(1-\rho_{-1}^j-\rho_{+1}^j)^{2}},\forall j\in[q]$.
	\begin{lemma}\label{lemma:rl_rcd}
		For any $\delta>0$, with probability at least $1-\delta$,
		\begin{equation*}
		\max_{\f\in\mathcal{F}}\left\vert\hR_{\tL}(\f)-\R_{\tL}(\f)\right\vert\leq 2q(q-1)K_{\rho}\R_{n}(\F)+q(q-1)\mu\Theta\sqrt{\frac{\ln\frac{2}{\delta}}{2n}}
		\end{equation*}
	\end{lemma}
	We omit the proof, since it can be proved similarly to Lemma \ref{lemma:hl_rcd}.
	
	Finally, similar to theorem 2, based on lemma \ref{lemma:rl_rc} and \ref{lemma:rl_rcd}, it is easy to obtain theorem 4. 
\end{proof}

\section{Proof for Theorem 5}

\begin{proof}
	For notational simplicity, we introduce the following notation:
	\begin{equation*}
	\Delta_{j,k}=\sum_{\y:y_j=-1,y_k=+1}p_{\y}
	\end{equation*}
	and we have
	\begin{equation}\label{eq:delta}
	\Delta_{j,t}\leq\Delta_{t,j} \quad\text{if}\quad \Delta_{j,k}\leq\Delta_{k,j} \quad\text{and}\quad \Delta_{k,t}\leq\Delta_{t,k}
	\end{equation}
	which has been proved in \cite{gao2013consistency}.
	
	Let us first define the set of Bayes predictors with respective to ranking loss
	\begin{lemma}\label{lemma:bayes_pred}
		\cite{gao2013consistency} For every $\p\in\Lambda$, where $\Lambda=\{\p\in\mathbb{R}^{|\mathcal{Y}|}:\sum_{\y\in\mathcal{Y}}p_{\y}=1 \hspace{0.5em}\text{and}\hspace{0.5em} p_{\y}\geq 0\}$ denotes the flat simplex of  $\mathbb{R}^{\vert\mathcal{Y}\vert}$, the set of Bayes predictions with respective to ranking loss is given by
		\begin{equation*}
		\mathcal{A}(\p)= \{\f: \forall j<k, f_j>f_k\quad \text{if}\quad \Delta_{j,k}<\Delta_{k,j}; f_j<f_k\quad \text{if}\quad \Delta_{j,k}>\Delta_{k,j}\}
		\end{equation*}
	\end{lemma}
	
	For modified surrogate loss $\tL$, we have
	\begin{align}\label{eq:rl_con_risk}
	\widetilde{W}(\p,\f)&=\sum_{\y\in\mathcal{Y}}p_{\y}\E_{\ty|\y}[\f(\x),\ty]=\sum_{\y\in\mathcal{Y}}p_{\y}\sum_{y_j=-1,y_k=+1}\phi(f_k-f_j)\\
	& = \sum_{1\leq j<k\leq q}\Delta_{j,k}\phi(f_k-f_j)+\Delta_{k,j}\phi(f_j-f_k)\nonumber
	\end{align}
	
	Following Theorem 10 in \cite{gao2013consistency}, if $\phi$ is a differential and non-increasing function with $\phi'(0)<0$ and $\phi(t)+\phi(-t)=2\phi(0)$, it suffices to prove that $f_j>f_k$ if $\Delta_{j,k}<\Delta_{k,j}$ for every $\f$ such that $\widetilde{W}^{*}(\p)=\widetilde{W}(\p,\f)$, where $\widetilde{W}^{*}(\p)=\inf_{\f}\widetilde{W}(\p,\f)$.Therefore, by minimizing $\widetilde{W}(\p,\f)$, we obtain the Bayes decision function $\f$, which proves $\tL_r$ is consistent to ranking loss. Accordingly, based on Corollary 25 in \cite{zhang2004consistency}, we have
	\begin{equation*}
	R_{L_r}(\hat\f)-R_{L_r}^{*}\leq\xi(R_{\L_{r}}(\hat\f)-R_{\L_{r}}^{*})
	\end{equation*}
	where $\xi$ is a non-negative concave function with $\xi(0)=0$. 
\end{proof}

\section{Details of Experimental Datasets}

\begin{table*}[!thb]
	\centering
	\small
	\caption{Characteristics of the experimental data sets.}   \label{table:data}
	\begin{tabular}{c|c|c|c|c|c|c}
		\hline
		\hline
		
		\textbf{Data set} & \textbf{\# Instances} & \textbf{\# Features} & \textbf{\# Class Labels} & \textbf{Cardinality} & \textbf{Batch Size} & \textbf{Domain}\\
		\hline
		\textbf{music\_emotion}  & 6833 & 98 & 11 & 2.42 & 400 & music \\
		\hline
		\textbf{music\_style}  & 6839 & 98 & 10 & 1.44 & 400 &music \\
		\hline
		\textbf{mirflickr}  & 10433 & 100 & 7 & 1.77 & 1000 & image \\
		\hline
		\textbf{enron} & 1702 & 1001 & 53 & 3.37 & 100 & text \\
		\hline
		\textbf{scene} & 2407	 & 294 & 6 & 1.07 & 100 & image \\
		\hline
		\textbf{yeast} & 2417 & 103	 & 14 & 4.23 & 100 & biology \\	
		\hline
		\textbf{slashdot} & 3782 & 1079 & 22 & 1.18	 & 200 & text \\
		\hline
		\textbf{tmc}  & 21519 & 500 & 22 & 2.15 & 1000 & text \\
		\hline
		\hline
	\end{tabular}
\end{table*}

\end{document}